\documentclass[letterpaper, 10 pt, conference, twocolumn]{ieeeconf}  

\IEEEoverridecommandlockouts                              

\usepackage[pdftex, pdfstartview={FitV}, pdfpagelayout={TwoColumnLeft},bookmarksopen=true,plainpages = false, colorlinks=true, linkcolor=black, citecolor = black, urlcolor = black,filecolor=black , pagebackref=false,hypertexnames=false, plainpages=false, pdfpagelabels ]{hyperref}

\usepackage[authormarkuptext=name,addedmarkup=bf,authormarkupposition=right]{changes}
\definechangesauthor[name={BTL}, color={blue}]{bl}

\usepackage{setspace}
\usepackage{graphicx}
\usepackage{epstopdf}
\usepackage{amssymb,amsmath}

\usepackage{amsthm}
\usepackage{mathtools}
\usepackage{cuted}
\usepackage[font=footnotesize]{subcaption}

\usepackage[font=footnotesize]{caption}
\usepackage{xcolor}
\usepackage{units}
\usepackage{algorithm}
\usepackage[noend]{algpseudocode}
\usepackage{balance}
\usepackage[sort,compress,noadjust]{cite}
\usepackage{tabularx}
\usepackage{nameref}
\usepackage{centernot}

\usepackage{arydshln}
\let\labelindent\relax
\usepackage{enumitem}

\theoremstyle{plain}
\newtheorem{theorem}{Theorem}
\theoremstyle{plain}
\newtheorem{proposition}{Proposition}
\theoremstyle{plain}
\newtheorem{lemma}{Lemma}
\theoremstyle{plain}
\newtheorem{corollary}{Corollary}
\theoremstyle{definition}

\theoremstyle{definition}
\newtheorem{definition}{Definition}
\theoremstyle{remark}
\newtheorem{remark}{Remark}

\usepackage[hang,flushmargin]{footmisc}

\usepackage[capitalize]{cleveref}
\crefformat{equation}{(#2#1#3)}
\Crefformat{equation}{Equation~(#2#1#3)}
\Crefname{equation}{Equation}{Eqs.}

\usepackage{accents}
\newcommand{\ubar}[1]{\underaccent{\bar}{#1}}

\newcommand{\qvec}[1]{\left[ \begin{array}{c} 0 \\ #1 \end{array} \right]}
\newcommand{\quat}[2]{\left[\begin{array}{c} {#1} \\ {#2} \end{array} \right]}

\title{\LARGE \bf
A Contracting Hierarchical Observer for Pose-Inertial Fusion
}

\author{Brett T. Lopez
\thanks{Verifiable and Control-Theoretic Robotics Laboratory, University of California, Los Angeles, Los Angeles CA, {\tt\small btlopez@ucla.edu}}
}

\begin{document}

\maketitle
\thispagestyle{empty}
\pagestyle{empty}


\begin{abstract}
This work presents a contracting hierarchical observer that fuses position and orientation measurements with an IMU to generate smooth position, linear velocity, orientation, and IMU bias estimates that are guaranteed to converge to their true values. 
The proposed approach is composed of two contracting observers.
The first is a quaternion-based orientation observer that also estimates gyroscope bias.
The output of the orientation observer serves as an input for another contracting observer that estimates position, linear velocity, and accelerometer bias thus forming a hierarchy.
We show that the proposed observer guarantees all state estimates converge to their true values. 
Simulation results confirm the theoretical performance guarantees.
\end{abstract}


\section{INTRODUCTION}
\label{sec:intro}

Many mobile robot platforms utilize a software architecture that consists of two state estimation schemes: one focused on position and orientation (pose) estimation accuracy and the other on generating smooth estimates for control and motion planning.
In practice, this decoupled architecture entails having an upstream vision or LiDAR odometry algorithm, which may or may not utilize inertial measurements, that generates an accurate pose estimate.
The pose generated by this upstream module is then combined with high-rate IMU measurements via a downstream fuser such an extended Kalman filter (EKF) or pose graph optimization (PGO) to generate a state estimate suitable for feedback control.
Despite the widespread use of this architecture, little focus is placed on the performance / convergence properties of the downstream fuser despite its critical role for control.
This article will present a pose-inertial measurement fusion scheme, rooted in contraction theory, that is \emph{globally exponentially convergent} --- a key property used to establish theoretical and practical performance guarantees for real-world sensor fusion on mobile robots.

The main challenge of developing a pose-inertial fuser with convergence and robustness guarantees is the nonlinearities that arise from 1) the orientation dynamics and 2) the coupling between orientation and translation of body-mounted inertial sensors.
The EKF, and its variants like the error-state EKF \cite{roumeliotis1999circumventing,sola2017quaternion}, is the most popular scheme for handling these types of nonlinearities via linearization.
However, the linearization process introduces numerical instabilities that require modifying the baseline algorithm in addition to extensive parameter tuning to get adequate performance, let alone convergence guarantees.
More recently, pose graph optimization \cite{dellaert2012factor,dellaert2021factor} has become a computationally tractable and more accurate alternative to filtering-based approaches, but requires a good initial guess for the nonlinear solver to converge to a reasonable estimate.
Moreover, discontinuous estimates can occur when a PGO is used to generate a real-time estimates as the solver may jump to a new local minimum given new sensor measurements.

The main contribution of this work is the development and analysis of a contracting hierarchical observer that guarantees convergence of position, linear velocity, orientation, and IMU bias estimates to their true values when position and orientation measurements are available, e.g., from an upstream algorithm. 
The proposed observer consists of a quaternion-based observer whose estimate is an input to another observer that estimates position, linear velocity, and accelerometer bias.
The orientation observer captures the underlying topology of the Lie group formed by quaternions and systematically addresses the unwinding phenomenon \cite{mayhew2011quaternion} that arises from the double cover property of quaternions.
The translation observer relies on the translational states being uniformly observable for any orientation --- a property that greatly simplifies observer design since uniform observability guarantees the existence of a contracting observer and also provides constructive conditions for the observer gain.
Despite the hierarchical structure, the proposed approach possesses strong convergence and robustness properties with minimal computation complexity making it ideal for resource-constrained systems that need smooth but accurate state estimates for control or motion planning. 
Simulation results confirm the theoretical convergence properties of the proposed approach.

\textit{Notation:} Symmetric positive definite $n\times n$ matrices are denoted as $\mathcal{S}^n$.
Positive and strictly-positive scalars are designated as $\mathbb{R}_+$ and $\mathbb{R}_{>0}$ respectively. 
The $n$-dimensional identity matrix is represented as $I_n$.
If $a,~b \in \mathbb{R}^3$ then the operator $[\cdot]_\times: \mathbb{R}^3 \rightarrow \mathbb{R}^{3\times 3}$ is the skew-symmetric matrix where $ a \times b = [a]_\times b$.
The matrix version of the Lie derivative is expressed as $L_A C = \dot{C} + C A$ where $C \in \mathbb{R}^{p \times n}$ and $A \in \mathbb{R}^{n\times n}$.
The standard notation for higher-order derivatives of the Lie derivative still hold, i.e., $L^2_A C = \tfrac{d}{dt}({L_A C}) + L_A C\,  A$ and so on.


\section{PROBLEM FORMULATION \&  BACKGROUND}
\label{sec:formulation}
This work is concerned with estimating the position $p\in\mathbb{R}^3$, linear velocity $v \in \mathbb{R}^3$, orientation (represented as a unit quaternion, see \nameref{sec:appendix}) $q \in \mathbb{S}^3$, and inertial measurement unit (IMU) bias of a mobile system where its position and orientation in an inertial coordinate frame are available.
The IMU provides high-rate body-fixed translational acceleration $a_m \in \mathbb{R}^3$ and angular velocity $\omega_m \in \mathbb{R}^3$ measurements corrupted by static accelerometer bias $\prescript{a}{}{b} \in \mathbb{R}^3$ and gyroscope bias $\prescript{g}{}{b} \in \mathbb{R}^3$.
The nonlinear rigid body kinematics under consideration are
\begin{equation}
\label{eq:system}
    \begin{aligned}
        \dot{p} & = v \\
        \dot{v} & = R(q)[a_m - \prescript{a}{}{b}] - g \\
        \prescript{a}{}{}\dot{b} & = 0 \\
        \dot{q} & = \tfrac{1}{2} q \otimes \qvec{\omega_m - \prescript{g}{}b} \\
        \prescript{g}{}{}\dot{b} & = 0
    \end{aligned}
\end{equation}
where $R \in \mathbb{SO}(3)$ is the rotation matrix formed by the quaternion $q$ and $g \in \mathbb{R}^3$ is the gravity vector.
\cref{eq:system} is a very common model used in aerospace and robotics to describe the translation and rotation of a mobile platform, e.g., multirotor or wheeled/tracked robot, equipped with an IMU.
In practice, \cref{eq:system} is numerically integrated with every new IMU measurement to generate a high-rate state estimate that can be used in feedback or motion planning.
However, it is well-known that pure integration of \cref{eq:system} will result in drift, which can be considerable, so external measurements from a camera or LiDAR are necessary to anchor the state estimate near its true value. 

The primary objective of this work is to develop an estimation methodology that generates smooth estimates that converge to their true values for each state in \cref{eq:system} where position and orientation measurements are available via an upstream algorithm focused on generating an accurate pose.
As stated in the Introduction, this upstream / downstream estimation architecture is common in aerospace and robotics since generating an accurate pose estimate is often at odds with generating a smooth estimate suitable for control.
A system diagram of the scenario of interest is shown in \cref{fig:system-diagram} where the smooth estimates generated by our approach can be used for control, motion planning, or even as a prior to warm start the upstream pose estimation algorithm \cite{chen2022direct}.
The proposed strategy will leverage the hierarchical structure of \cref{eq:system}, namely that the orientation kinematics (including the gyroscope bias) can be decoupled from the translation kinematics (including the accelerometer bias), i.e., the last two equations of \cref{eq:system} are independent of the first three equations of \cref{eq:system}.

\begin{figure}[t!]
     \centering
     \includegraphics[width=0.9\columnwidth]{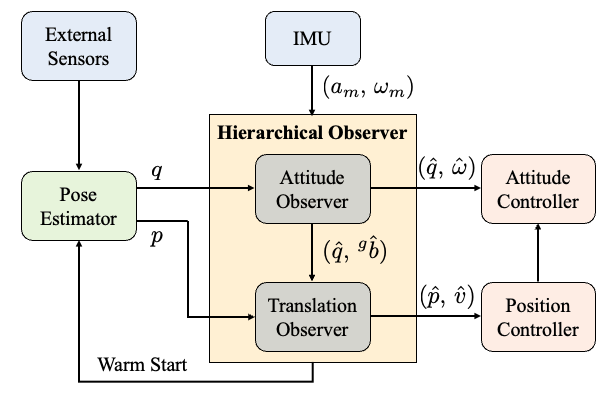}
     \caption{General system architecture where the proposed hierarchical observer fuses pose and IMU measurements to generate a high-rate state estimate suitable for feedback control.}
     \label{fig:system-diagram}
     \vskip -0.2in
\end{figure}

Central to our approach is contraction analysis \cite{lohmiller1998contraction,wang2005partial,manchester2017control}, an alternative to Lyapunov stability analysis that has led to several foundational results in dynamical systems theory, nonlinear control, adaptive control, motion planning, and learning, to name a few. 
Fundamentally, the distinguishing characteristic of a contracting system is that the distance between any two arbitrary system trajectory will exponentially shrink to zero.
This property is extremely useful in many situations since it is often desirable for trajectories to converge to each other, e.g., trajectory tracking or estimation, rather than to an equilibrium point.
A necessary and sufficient condition for a dynamical system $\dot{x} = f(x,t)$ with $x \in \mathbb{R}^n$ and $f: \mathbb{R}^n \times \mathbb{R} \rightarrow \mathbb{R}^n$ to be contracting is the existence of a metric $M: \mathbb{R}^n \times \mathbb{R} \rightarrow \mathcal{S}^n$ such that
\begin{equation}
    \label{eq:contraction}
    \frac{\partial f}{\partial x}^\top M + M \frac{\partial f}{\partial x} + \dot{M} \preceq - 2 \lambda M,
\end{equation}
where $\ubar{\alpha} I_n \preceq M \preceq \bar{\alpha} I_n$, $\dot{M} = \sum_{i=1}^n \partial M / \partial x_i \, \dot{x}_i + \partial M / \partial t$, and $\lambda \in \mathbb{R}_{>0}$ is the contraction (convergence) rate.
If \cref{eq:contraction} is satisfied then for any two arbitrary $x_1(t)$ and $x_2(t)$ they satisfy $\|x_2(t)-x_1(t)\| \leq \sqrt{\nicefrac{\bar{\alpha}}{\ubar{\alpha}}} \|x_2(0)-x_1(0)\| e^{-\lambda t}$ which yields $\|x_2(t)-x_1(t)\| \rightarrow 0$ exponentially at rate $\lambda$ \cite{lohmiller1998contraction}.
We will make use of \cref{eq:contraction} and the properties it implies when developing and analyzing the proposed hierarchical observer.

\section{CONTRACTING HIERARCHICAL OBSERVER}
\label{sec:observer}

\subsection{Overview}
Careful inspection reveals that \cref{eq:system} is a hierarchical system where the orientation kinematics are decoupled from the translation kinematics.
The hierarchical structure facilitates the design of a nonlinear observer that possesses strong convergence and robustness guarantees.
This section will first presents a nonlinear orientation observer that guarantees the quaternion and gyroscope bias estimates will converge to their true values.
Further, it will be shown that the orientation observer actually yields contracting dynamics for the error quaternion.
Then, making use of the strong convergence properties of the orientation observer, a contracting observer that estimates the position, linear velocity, and accelerometer bias will be proposed.
Finally, the convergence of the combined hierarchical observer will be established.

\subsection{Orientation Observer}
\label{sub:quaternion}

Quaternion estimation presents several unique challenges.
Conceptually, unit quaternions form a Lie group that lives on the three sphere $\mathbb{S}^3$ so any estimation scheme must properly account for the underlying topology of $\mathbb{S}^3$.
In other words, the quaternion estimation error must be formed by the quaternion product (denoted as $\otimes$) which is the operator of their Lie group.
Additionally, quaternions double cover $\mathbb{SO}(3)$ so $q$ and $-q$ represent the same orientation.
This can lead to the so-called unwinding phenomenon where the quaternion estimate takes longer than necessary to converge to its true value.
The proposed orientation observer addresses all the aforementioned points.
The following theorem presents the orientation observer and proves it convergence guarantees.

\begin{theorem}
\label{thm:att-obs}
    Let the true quaternion be $q$ and the error quaternion be $q_e \triangleq \hat{q}^* \otimes q$.
    The quaternion and gyroscope bias estimates $\hat{q}$ and $\prescript{g}{}{}\hat{b}$ will globally converge to their true values with the geometric orientation observer
    \begin{subequations}
    \label{eq:geo-att}
        \begin{align}
            \dot{\hat{q}} & = \tfrac{1}{2} \hat{q} \otimes \left( \qvec{\omega_m - \prescript{g}{}{}\hat{b}} + 2\,c_1 \left[ \begin{array}{c} 1 - |q_e^\circ| \\ \mathrm{sgn}(q_e^\circ) \, \vec{q}_e  \end{array} \right] \right) \label{eq:geo-q}  \\
            \prescript{g}{}{}\dot{\hat{b}} &= - c_2 \, q_e^\circ \, \vec{q}_e, \label{eq:geo-bg}
        \end{align}
    \end{subequations}
    where $c_1,~c_2 \in \mathbb{R}_{>0}$ and $\omega_m$ is the corrupted angular velocity measurements from an IMU. 
\end{theorem}

\begin{proof}
From the definition of the error quaternion we can derive the error quaternion kinematics
\begin{equation}
\label{eq:qerr_kin}
    \begin{aligned}
        \dot{q}_e = & ~ \hat{q}^* \otimes \dot{q} - \hat{q}^* \otimes \dot{\hat{q}} \otimes q_e \\
                  = & ~ \tfrac{1}{2} q_e \otimes \qvec{\omega_m - \prescript{g}{}b} - \tfrac{1}{2} \qvec{\omega_m -\prescript{g}{}{}\hat{b}} \otimes q_e \\
                  & ~ - c_1 \quat{1-|q_e^\circ|}{\mathrm{sgn}(q_e^\circ)\,\vec{q}_e} \otimes q_e ,
    \end{aligned}
\end{equation}
where the geometric orientation observer was used in place of $\dot{\hat{q}}$.
Noting that 
\begin{equation*}
    \quat{1-|q_e^\circ|}{\mathrm{sgn}(q_e^\circ)\,\vec{q}_e} \otimes q_e = \quat{ (1-|q_e^\circ|) q_e^\circ - \mathrm{sgn}(q_e^\circ)\, \vec{q}_e^\top \vec{q}_e}{\vec{q}_e},
\end{equation*}
and defining the gyroscope bias estimate error to be $\prescript{g}{}b_{e} \triangleq \prescript{g}{}b - \prescript{g}{}{}\hat{b}$, then the vector part of \cref{eq:qerr_kin} is 
\begin{equation}
\label{eq:qerr_vec_kin}
    \dot{\vec{q}}_e = \vec{q}_e \times \omega_m - \tfrac{1}{2} q_e^\circ \prescript{g}{}b_e - \vec{q}_e \times (\prescript{g}{}b + \prescript{g}{}{}\hat{b}) - c_1 \, \vec{q}_e.
\end{equation}
Now consider the Lyapunov function
\begin{equation*}
    V(\vec{q}_e,\prescript{g}{}b_e) = \|\vec{q}_e\|^2 + \tfrac{1}{2 c_2} \|\prescript{g}{}b_e\|^2.
\end{equation*}
Differentiating and using \cref{eq:qerr_vec_kin},
\begin{equation*}
    \begin{aligned}
        \dot{V} = & ~ 2 \vec{q}_e^\top \Big( \vec{q}_e \times \omega_m - \tfrac{1}{2} q_e^\circ \prescript{g}{}b_e - \vec{q}_e \times (\prescript{g}{}b + \prescript{g}{}{}\hat{b}) - c_1 \, \vec{q}_e \Big)  \\
        & ~ - \tfrac{1}{c_2} \prescript{g}{}b_e^\top \prescript{g}{}{}\dot{\hat{b}} \\
        = & ~ - 2 \, c_1 \, \|\vec{q}_e\|^2 - \prescript{g}{}b_e^\top \left( \tfrac{1}{c_2} \prescript{g}{}{}\dot{\hat{b}} + q_e^\circ \vec{q}_e \right).
    \end{aligned}
\end{equation*}
Using \cref{eq:geo-bg} then yields $\dot{V} (\vec{q}_e,\prescript{g}{}b_e) = -2 \, c_1\, \|\vec{q}_e\|^2 \leq 0$ which shows that $\vec{q}_e$ and $\prescript{g}{}b_e$ are bounded.
However, $\dot{V}(\vec{q}_e,\prescript{g}{}b_e)$ is only negative semidefinite since $\dot{V} (\vec{0},\prescript{g}{}b_e) = 0$ for any $\prescript{g}{}b_e$. 
Noting that \cref{eq:qerr_vec_kin} is a non-autonomous system due to the presence of $\omega_m$, we can establish convergence via Barbalat's lemma.
Letting $W(\vec{q}_e) = \|\vec{q}_e\|^2$, integrating $\dot{V}(\vec{q}_e,\prescript{g}{}b_e)$ yields $2\,c_1\,\int_0^\infty W(\vec{q}_e(\tau))\,d\tau = V(\vec{q}_e(0),\prescript{g}{}b_e(0)) < \infty$. 
Moreover, $\dot{W}(\vec{q}_e) = -q_e^\circ \vec{q}_e^\top \prescript{g}{}b_e - 2 c_1 \|\vec{q}_e\|^2$ is bounded since $q_e$ and $\prescript{g}{}b_e$ are bounded (as stated above). 
Then, by Barbalat's lemma,  $W(\vec{q}_e) \rightarrow 0 \implies \vec{q}_e \rightarrow \vec{0}$ as $t\rightarrow+\infty$. 
Further, in order for the set $\mathcal{M} = \{ (\vec{q}_e,\prescript{g}{}b_e) \, : \, \vec{q}_e = \vec{0} \}$ to be invariant, $\prescript{g}{}b_e$ must also be zero since $(\dot{\vec{q}}_e,\,\dot{{b}}^{\,g}_e) = (\vec{0},\,\vec{0}) \iff (\vec{q}_e,\,\prescript{g}{}b_e) = (\vec{0},\,\vec{0})$.
Therefore, $(\vec{q}_e,\,\prescript{g}{}b_e) \rightarrow (\vec{0},\,\vec{0})$ as $t \rightarrow +\infty$ so $\hat{q} \rightarrow q$ and $\prescript{g}{}{}\hat{b} \rightarrow \prescript{g}{}b$ asymptotically\footnote{Technically $\hat{q}$ may converge to $q$ or $-q$ (which ever is closer) but these represent the same orientation due to the double covering property of quaternions so this statement is without loss of generality.}.
Since the above holds for all $\vec{q}_e$ and $\prescript{g}{}b_e$ then convergence is also global.
\end{proof}

\begin{remark}
While the orientation observer \cref{eq:geo-att} contains a discontinuity, it actually ensures 1) the estimate $\hat{q}$ \emph{remains continuous} even if $q$ changes sign and 2) $q_e$ converges to the \emph{closest} equilibrium, namely $q_e = (\,\pm 1, \vec{0}\,)$, thereby directly addressing the unwinding phenomenon.
In fact, one can show that $\mathrm{sgn}(q_e^\circ)\vec{q}_e$ will never change sign even if $q$ changes sign.
Hence, \cref{eq:geo-att} will not exhibit any discontinuous estimates even with the presence of $\mathrm{sgn}(\cdot)$.
\end{remark}

The convergence result in \cref{thm:att-obs} can be further strengthened by noting that the error quaternion dynamics with the observer \cref{eq:geo-att} are contracting.

\begin{theorem}
    \label{thm:qe_contracting}
    The error quaternion dynamics with the observer \cref{eq:geo-att} are contracting with rate $c_1$ in the identity metric.
\end{theorem}

\begin{proof}
    It is sufficient to show that the vector part of the error quaternion dynamics are contracting because $q_e$ must satisfy $q_e^\circ(t)^2 + \|\vec{q}_e(t)\|^2 = 1$ for all $t$ so if $\|\vec{q}_e\| \rightarrow 0$ at rate $c_1$ then $|q_e^\circ| \rightarrow 1$ at the same rate.
    The vector part of the error quaternion dynamics in \cref{eq:qerr_vec_kin} can be rewritten as
    \begin{equation}
        \label{eq:qerr_matrix}
        \dot{\vec{q}}_e = [\omega_m + \prescript{g}{}b + \prescript{g}{}{}\hat{b}]_\times \vec{q}_e - c_1 \, \vec{q}_e - \tfrac{1}{2} q_e^\circ \prescript{g}{}b_e,
    \end{equation}
    where $[\cdot]_\times$ is a skew-symmetric matrix and $\tfrac{1}{2} q_e^\circ \prescript{g}{}b_e$ is an exogenous input.
    Let $J$ be the Jacobian of \cref{eq:qerr_matrix} with respect to $\vec{q}_e$. 
    Using the property $[\cdot]_\times^\top = -[\cdot]_\times$, the symmetric part of the Jacobian satisfies $J + J^\top = -2c_1 \, I_3$ which is equivalent to the contraction condition \cref{eq:contraction} where the metric is identity and rate of contraction is $c_1$.
\end{proof}

An immediate consequence of \cref{thm:qe_contracting} is that the quaternion estimate $\hat{q}$ will exponentially converge to a region near $q$ if the gyroscope bias is bounded.

\begin{corollary}
Let the gyroscope bias belong to a closed bounded set $\prescript{g}{}{\mathcal{B}}$. 
The quaternion estimate $\hat{q}$ will globally converge exponentially with rate $c_1$ to a region near the true quaternion $q$ with the orientation observer \cref{eq:geo-att}.
\end{corollary}

\begin{proof}
Let $\prescript{g}{}{\mathcal{B}} \triangleq \left\{ \prescript{g}{}b_e \in \mathbb{R}^3 \, : \, \|\prescript{g}{}b_e\| \leq \prescript{g}{}{}\bar{b}_e < \infty  \right\}$. 
Since the dynamics for $\vec{q}_e$ are contracting in the identity metric at rate $c_1$, it follows that
\begin{equation*}
\begin{aligned}
    \tfrac{d}{dt}\|\vec{q}_e \|^2 = &  - 2 \, c_1 \, \|\vec{q}_e \|^2 + q_e^\circ \vec{q}_e^\top \prescript{g}{}b_e \\ 
    \leq &  -2 \, c_1 \, \|\vec{q}_e \|^2 + |q_e^\circ| \, \|\vec{q}_e\| \, \prescript{g}{}{}\bar{b}_e \\
    \leq & -2 \, c_1 \, \|\vec{q}_e \|^2 + \, \|\vec{q}_e\| \, \prescript{g}{}{}\bar{b}_e,
\end{aligned}
\end{equation*}
where the last inequality uses the identity $\|q\| = 1 \implies |q^\circ| \leq 1$.
The above expression can be further simplified to $\tfrac{d}{dt} \|\vec{q}_e\| \leq - c_1 \|\vec{q}_e\| + \tfrac{1}{2} \prescript{g}{}{}\bar{b}_e$ which has the analytic solution $\|\vec{q}_e(t)\| \leq \|\vec{q}_e(0)\| e^{-c_1 t} + \tfrac{1}{2c_1} \prescript{g}{}{}\bar{b}_e (1-e^{-c_1 t}).$
Hence, $\|\vec{q}_e(t)\|$ converges exponentially with rate $c_1$ to $\tfrac{1}{2c_1} \prescript{g}{}{}\bar{b}_e$.
Since $q_e = \hat{q}^* \otimes q$ then $\hat{q}$ must also converge exponentially with rate $c_1$ to a region near the true quaternion $q$.
\end{proof}

The implications of \cref{thm:att-obs,thm:qe_contracting} are three-fold. 
Firstly, the quaternion estimate $\hat{q}$ will exhibit two phases of convergence, namely, exponential convergence until $\|\vec{q}_e\| \leq \tfrac{1}{2c_1} \prescript{g}{}{}\bar{b}_e$ followed by asymptotic convergence to $q$.
Moreover, after the exponential convergence phase, $\hat{q}$ will \emph{always} be near $q$ even if $\prescript{g}{}{}{b}$ is slowly time-varying due to, e.g., temporal temperature changes so long as $\prescript{g}{}{}{b} \in \prescript{g}{}{\mathcal{B}}$.
This property is also true if other bounded exogenous signals enter the error quaternion dynamics.
It can thus be concluded that the proposed quaternion observer possesses an inherent robustness to external inputs, which is useful both practically and theoretically.
In particular, the inherent robustness of \cref{eq:geo-att} can perhaps be leveraged to improve the performance of the upstream pose estimation algorithm by providing a prior that could, e.g., serve as a warm start for a nonlinear optimization. 
Secondly, one can derive performance bounds --- in the mean squared sense --- for a contracting system when stochastic noise is present,  such as in the angular velocity or quaternion measurements.
A full stochastic analysis of \cref{eq:geo-att} will be conducted in future work.
Thirdly, the strong convergence properties of \cref{eq:geo-att} enables a decoupled estimation strategy for \cref{eq:system} without sacrificing performance. 

\subsection{Translation Observer}
\label{sub:translation}

The strong convergence properties of the proposed orientation observer in \cref{sub:quaternion}, in conjunction with the hierarchical structure of \cref{eq:system}, enable the decoupled estimation of orientation and translation states (including biases).
We will first treat the vehicle's orientation, represented as the rotation matrix $R$ formed from ${q}$, as a known time-varying signal that the translation observer must compensate for.
It will be shown that the proposed translation observer is contracting.
Further, we will prove that the translation estimation error will converge to zero even when the true orientation is replaced by its estimate generated by the orientation observer \cref{eq:geo-att}.
Thus, the overall proposed approach represents a contracting hierarchical observer for estimating orientation and translation.
Before proceeding, we revisit uniform observability and its connection with contraction.
\begin{definition}
\label{def:observability}
A linear time-varying system
\begin{equation}
\label{eq:ltv}
    \begin{aligned}
        \dot{x} &= A(t) x + B(t)v \\
        y &= C(t) x,
    \end{aligned}
\end{equation}
with state $x \in \mathbb{R}^n$, exogenous input $v \in \mathbb{R}^m$, measurement $y\in\mathbb{R}^p$, and time-varying matrices $A: \mathbb{R} \rightarrow \mathbb{R}^{n\times n}$, $B: \mathbb{R} \rightarrow \mathbb{R}^{n\times m}$, $C: \mathbb{R} \rightarrow \mathbb{R}^{p\times n}$ is \emph{uniformly observable} if and only if the observability matrix
\begin{equation}
    \label{eq:unif_observ}
    \mathcal{O}(t) \triangleq \left[ \begin{array}{c} C \\ L_A C \\ L_A^2 C \\ \vdots \\ L_A^{n-1} C
    \end{array} \right],
\end{equation}
is full rank for all $t$.
\end{definition}

Uniform observability is an important property that facilitates observer design. 
A fundamental property of uniformly observable linear systems is the existence of a diffeomorphism $\Upsilon:\mathbb{R} \rightarrow \mathbb{R}^{n\times n}$ so that if $z = \Upsilon(t) x$ then \cref{eq:ltv} can be transformed into observable canonical form \cite{zeitz1984observability}
\begin{equation}
\label{eq:ocf}
    \begin{aligned}
        \dot{z} &= \left[ \begin{array}{cc} 0 & 0 \\ I_{n-p} & 0 \end{array} \right] z - \sigma(t)y + \Upsilon(t)B(t)v \\
        & = A_{{o}} - \sigma(t)y + B_{o}(t)v \\
        y &= [\,0~0~\cdots~I_p\,]\,z = C_{o} \, z,
    \end{aligned}
\end{equation}
where $\sigma : \mathbb{R} \rightarrow \mathbb{R}^{n\times p}$ and the zero elements are of appropriate dimensions.
The main benefit of representing \cref{eq:ltv} in observable canonical form is that the time-varying terms depend solely on \emph{known} quantities, namely, $\sigma(t)$ and $y$.
A suitable observer for \cref{eq:ocf} takes the form
\begin{equation}
    \label{eq:ocf-observer}
    \dot{\hat{z}} = A_o \hat{z} - \sigma(t)y + B_o(t)v + K_{o}  (y-\hat{y})
\end{equation}
where $K_{o} \in \mathbb{R}^{n \times p}$.
If $z_e \triangleq z - \hat{z}$ then the $z_e$ dynamics become the linear \emph{time-invariant} system 
\begin{equation}
\label{eq:ocf-error}
    \dot{z}_e = ( A_{o} - K_{o} C_{o}  ) z_e,
\end{equation}
where standard pole placement or other gain selection techniques for linear time-invariant systems can be employed, making the observer design straightforward.
Furthermore, since $\Upsilon(t)^{-1}$ exists (by definition) then $x_e = \Upsilon(t)^{-1} z_e$ so if $z_e \rightarrow 0$ then $x_e \rightarrow 0$.
Observer \cref{eq:ocf-observer} in $x$-coordinates takes the form $\dot{\hat{x}} = A(t)\hat{x} + B(t) v + K(t)(y-\hat{y})$ with the time-varying gain $K(t) \triangleq \Upsilon(t)^{-1} K_{o}$.

In \cite{lohmiller1998contraction}, uniform observability was shown to be a sufficient condition for contracting estimation error dynamics.
The argument leverages two key properties of uniformly observable systems, namely the guaranteed existence of the diffeomorphism $\Upsilon(t)$ and the freedom to place the poles of the estimation error dynamics anywhere. 
Note that the observer gain $K_{o}$ can be found by instead searching for a positive definite matrix $P \in \mathcal{S}^{n}$ and strictly-positive scalar $\rho \in \mathbb{R}_{>0}$ such that the implication \cite{manchester2014output}
\begin{equation*}
    C_o z_e = 0 \implies z_e^\top \left(A_o^\top P + A_o P + 2 \lambda P\right) z_e \leq 0,
\end{equation*}
 is satisfied. Or equivalently via Finsler's theorem
\begin{equation}
\label{eq:contraction-observer}
    A_o^\top P + A_o P + 2 \lambda P - \rho \, C_o C^\top_o \preceq 0,
\end{equation}
which in essence is a necessary and sufficient condition for the observer \cref{eq:ocf-observer} to be contracting in metric $P$ with rate $\lambda$. 
Once $P$ and $\rho$ are known, the observer gain can be directly computed via $K_o = \tfrac{1}{2} \rho P^{-1} C_o^\top$.
Further noting that $z_e = \Upsilon(t)x_e$, then the $x_e$ dynamics will also be contracting in the metric $M(t) = \Upsilon(t)^\top P \Upsilon(t)$ with the observer gain $K(t) =  \Upsilon(t)^{-1}K_o$ where the inverse of the diffeomorphism $\Upsilon(t)^{-1} = [\, \gamma_1(t)~\gamma_2(t) ~ \cdots ~ \gamma_n(t)\,]$ with column vectors $\gamma_i$ can be found via the recursion \cite{bestle1983canonical}
\begin{equation}
\label{eq:diffeo-recursion}
    \begin{aligned}
        \gamma_1(t) &= \mathcal{O}(t)^{-1} \left[ \begin{array}{c} 0 \\ \vdots \\ I_p
        \end{array} \right] \\
        \gamma_2(t) &= A(t)\gamma_1(t) - \dot{\gamma}_1(t) \\
        & \vdotswithin{=} \\
        \gamma_n(t) &= A(t)\gamma_{n-1}(t) - \dot{\gamma}_{n-1}(t). \\
    \end{aligned}
\end{equation}
Thus, uniform observability not only guarantees the existence of a contracting observer but also provides constructive conditions for a metric $M(t)$ and observer gain $K(t)$.
More formally, this leads to the following lemma.
\begin{lemma}
\label{lemma:contract-observability}
    If the linear time-varying system \cref{eq:ltv} is uniformly observable, then there exists $P \in \mathcal{S}^{n}$ and $\rho \in \mathbb{R}_{>0}$ satisfying \cref{eq:contraction-observer} leading to the observer
    \begin{equation}
        \label{eq:x-contract-observer}
        \dot{\hat{x}} = A(t)\hat{x} + B(t) v + K(t)(y-\hat{y}),
    \end{equation}
    that is contracting at rate $\lambda$ in the metric $M(t) = \Upsilon(t)^\top P \Upsilon(t)$ with gain $K(t) = \tfrac{1}{2} \rho \Upsilon(t)^{-1} P^{-1} \Upsilon(t)^{-\top} C^\top $ where $\Upsilon(t)^{-1}$ is given by \cref{eq:diffeo-recursion}.
\end{lemma}
\begin{proof}
    Omitted for brevity but follows from above.
\end{proof}

Based on the discussion above, the natural first step in designing an observer is to check if the observability matrix in \cref{def:observability} is full rank. 
The following proposition shows that the translation kinematics are uniformly observable for any time-varying rotation matrix $R(t)$. 
\begin{proposition}
\label{prop:trans-observable}
The time-varying translation kinematics
    \begin{equation*}
           \left[ \begin{array}{c} \dot{p} \\ \dot{v} \\ \prescript{a}{}{\dot{b}} \end{array} \right] = \left[ \begin{array}{ccc} 0 & I_3 & 0 \\ 0 & 0 & -{R}(t) \\ 0 & 0 & 0 \end{array} \right] \left[ \begin{array}{c} p \\ v \\ \prescript{a}{}{b} \end{array} \right] + \left[ \begin{array}{c} 0 \\ R(t) a_m - g \\ 0 \end{array} \right],
    \end{equation*}
with the measurement model $y = p$ and any time-varying rotation matrix ${R}(t)$ is uniformly observable.
\end{proposition}
\begin{proof}
    Let 
    \begin{equation*}
        A(t) = \left[ \begin{array}{ccc} 0 & I_3 & 0 \\ 0 & 0 & -{R}(t) \\ 0 & 0 & 0 \end{array} \right], ~~~ C = [\,I_3~~0~~0\,].
    \end{equation*}
    Noting that $A(t)$ is a block matrix, then, with a slight abuse of notation, $L_A C = \dot{C} + C A = [\,0~~I_3~~0\,]$ and $L_A^2 C = \tfrac{d}{dt}(L_A C) + L_A C A = [\,0~~0~-{R}(t)\,]$ so
    \begin{equation*}
        \mathcal{O}(t) = \left[ \begin{array}{ccc} I_3 & 0 & 0 \\ 0 & I_3 & 0 \\ 0 & 0 & -{R}(t) \end{array} \right].
    \end{equation*}
    Since $\mathcal{O}(t)$ is block diagonal and $\det(R) = 1$ for any rotation matrix, then $\det(\mathcal{O}(t)) = -1$ so $\mathcal{O}(t)$ is full rank and the translation dynamics with position measurements is uniformly observable.
\end{proof}

With uniform observability established in \cref{prop:trans-observable} for the translation kinematics, a contracting observer can be constructed such that the estimation error will converge to zero exponentially.
\begin{theorem}
\label{thm:geo-trans-true}
    Let the true position be $p$ and the position estimation error be $p_e \triangleq p - \hat{p}$.
    The position, linear velocity, and accelerometer bias estimates will exponentially converge to their true values with the observer
    \begin{equation}
    \label{eq:geo-trans-true}
        \begin{aligned}
            \dot{\hat{p}} &= \hat{v} + K_3 p_e \\
            \dot{\hat{v}} &= R(t)[a_m-\prescript{a}{}{}{\hat{b}}] - g \\
            &\hphantom{=} + [K_2 + K_3 \, R(t) \, \Omega(t) \, R(t)^\top] p_e \\
            \prescript{a}{}{}\dot{\hat{b}} &= -[K_1 + K_2\Omega(t) + K_3 (\Omega(t)^2 - \dot{\Omega}(t))] {R}(t)^\top p_e ,
        \end{aligned}
    \end{equation}
    where each $K_i$ is a diagonal matrix with positive entries, $R(t)$ is the true time-varying rotation matrix, and $\Omega(t) \triangleq [\omega]_\times$ is the true time-varying skew symmetric matrix formed with the true angular velocity vector $\omega$. 
    Furthermore, the observer \cref{eq:geo-trans-true} yields contracting estimation error dynamics.
\end{theorem}

\begin{proof}
Since the translation kinematics with position measurements is uniformly observable via \cref{prop:trans-observable}, there exists a $P$ and $\rho$ that satisfy the contraction condition \cref{eq:contraction-observer} for a rate $\lambda$. 
Due to the structure of $A_o$, one can show $K_o = \frac{1}{2} \rho P^{-1} C_o^\top = [K_1~K_2~K_3]^\top$ where $K_i = k_i \, I_3$ with $k_i > 0$.
From \cref{lemma:contract-observability}, the observer with gain $K(t) = \Upsilon(t)^{-1} K_o$ yields contracting estimation error dynamics where $\Upsilon(t)^{-1}$ takes the form
\begin{equation*}
    \Upsilon(t)^{-1} = \left[ \begin{array}{ccc}
         0 & 0 & I  \\
         0 & I & R(t)\Omega(t)R(t)^\top \\
         -R(t)^\top & -\Omega(t)R(t)^\top & -\Gamma(t) {R}(t)^\top
    \end{array} \right],
\end{equation*}
with $\Gamma(t) = \Omega(t)^2 - \dot{\Omega}(t)$.
Making the appropriate substitutions, the observer in \cref{eq:geo-trans-true} is obtained.
Since \cref{eq:geo-trans-true} yields contracting estimation error dynamics then $(p_e,\,v_e,\,\prescript{a}{}{b}_e) \rightarrow (\vec{0},\,\vec{0},\,\vec{0})$ at rate $\lambda$.
Therefore, the $\hat{p}$, $\hat{v}$, and $\prescript{a}{}{}\hat{b}$ converge to their true values exponentially.
\end{proof}

Conceptually, the presence of $\omega$ and $\dot{\omega}$ in \cref{eq:geo-trans-true} is a consequence of the observer having to compensate for the time-varying nature of the rotation matrix $R$.
In essence, these are \emph{anticipation terms} that are necessary for the error dynamics to be contracting.
Despite the strong convergence result, the observer in \cref{eq:geo-trans-true} uses the true orientation $R$ and angular velocity $\omega$ of the system which are generally not available. 
However, if $R$ and $\omega$ are instead replaced with their estimates generated by the orientation observer \cref{eq:geo-att} then the translation states will still converge to their true values, as shown in the following theorem.

\begin{theorem}
\label{thm:geo-trans}
    Let the true position be $p$ and the position estimation error be $p_e = p - \hat{p}$.
     The position, linear velocity, and accelerometer bias estimates will converge to their true values with the observer
    \begin{equation}
    \label{eq:geo-trans}
        \begin{aligned}
            \dot{\hat{p}} &= \hat{v} + K_3 p_e \\
            \dot{\hat{v}} &= \hat{R}\,[a_m-\prescript{a}{}{}{\hat{b}}] - g + [K_2 + K_3 \, \hat{R} \, \hat{\Omega} \, \hat{R}^\top] p_e \\
            \prescript{a}{}{}\dot{\hat{b}} &= -[K_1 + K_2\hat{\Omega} + K_3 (\hat{\Omega}^2 - \dot{\hat{\Omega}})] \hat{R}^\top p_e ,
        \end{aligned}
    \end{equation}
    where each $K_i$ is a diagonal matrix with positive entries, $\hat{R}$ (formed from $\hat{q}$) and $\hat{\Omega} = [\,\omega_m - \prescript{g}{}{}\hat{b}\,]_\times$ are the orientation and gyroscope bias estimates generates by the observer \cref{eq:geo-att}. 
\end{theorem}

\begin{proof}
The observer \cref{eq:geo-trans} can be put into the same form as \cref{eq:geo-trans-true} but with additional terms that depend on $R_e$ and $\prescript{g}{}{b}_e$ where $q_e = \hat{q}^* \otimes q \implies R_e = \hat{R}^\top R$ and $\prescript{g}{}{b}_e = \prescript{g}{}{b} - \prescript{g}{}{}\hat{b}$.
These additional terms can be viewed as exogenous inputs that tend to zero since the orientation observer \cref{eq:geo-att} guarantees $(\vec{q}_e,\,\prescript{g}{}{b}_e) \rightarrow (\vec{0},\,\vec{0})$ as $t \rightarrow +\infty$.
Furthermore, since \cref{thm:geo-trans-true} establishes the observer \cref{eq:geo-trans-true} is contracting, 
the observer \cref{eq:geo-trans} can be viewed as a contracting system with exogenous inputs that tend to zero.
Since a contracting system will tend towards its nominal output when driven by inputs tending to zero, then $(p_e,\,v_e,\,\prescript{a}{}{b}_e) \rightarrow (\vec{0},\,\vec{0},\,\vec{0})$ as $t \rightarrow +\infty$.
Therefore, $\hat{p}$, $\hat{v}$, and $\prescript{a}{}{}\hat{b}$ will converge to their true values.
\end{proof}

\begin{corollary}
The orientation and translation state estimates will converge to their true values with the hierarchical observer given by \cref{eq:geo-att,eq:geo-trans}.
\end{corollary}
\begin{proof}
    Follows from \cref{thm:att-obs,thm:geo-trans}.
\end{proof}

\begin{remark}
\cref{eq:geo-trans} depends on $\dot{\omega}_m$ which is not directly available. 
However, one could numerically differentiate $\omega_m$ and apply a low-pass filter to reduce noise, or implement another observer / filter that treats $\omega_m$ as an external measurement.
\end{remark}





\section{SIMULATION RESULTS}
\label{sec:results}

The proposed observer was tested in simulation where synthetic IMU measurements were corrupted with static accelerometer and gyroscope bias. 
The system's true position and orientation was generated by numerically integrating the true accelerometer and gyroscope measurements $a = [ \sin(t)~2\sin(0.1\,t)~0.3 ]^\top$ m/s$^2$ and $\omega = [\sin(2\,t)~-\sin(4\,t)~2\sin(t)]^\top$ rad/s, respectively. 
Note that the true position and orientation are time-varying so the observer must converge to a trajectory rather than a constant position and orientation.
The initial state of the observer was initialized to a random value in order to showcase the observer's strong convergence properties.
The translation observer gain was computed numerically by formulating the search for the positive definite matrix $P$ and strictly-positive scalar $\rho$ as a linear matrix inequality that minimized the condition number of $P$ with \cref{eq:contraction-observer} as a constraint.  
All simulation parameters can be found in \cref{table:params} of the \nameref{sec:appendix}.

The observer's performance was quantified by evaluating the convergence behavior of $\|\vec{q}_e\|$ (shown in \cref{fig:qe}) and the norm of the translation estimation error $x_e \triangleq [p_e~v_e~\prescript{a}{}{b_e}]^\top$ with metric $M(t)$, i.e., $x_e^\top M(t) x_e$ (shown in \cref{fig:E}). 
Firstly, we see that $\|\vec{q}_e\|$ initially exhibits exponential convergence to $\tfrac{\prescript{g}{}{}{\bar{b}_e}}{2c_1}$ followed by asymptotic convergence to zero.
The black vertical line indicates the first instance when $\|\vec{q}_e\| \leq \tfrac{\prescript{g}{}{}{\bar{b}_e}}{2c_1}$.
This behavior confirms the analysis of the proposed orientation observer.
Secondly, we see that the translation estimation error $x_e$ converges to zero even when using the orientation estimate generated by the orientation observer.
In particular, after the initial transients attributed to the initial exponential convergence of $\|\vec{q}_e\|$, the translation estimation error exhibits nearly exponential convergence. 
This follows from the analysis in \cref{sub:translation} where we only established the translation estimation error will tend toward zero as the orientation estimate also tends to zero. 
Nonetheless, the estimation error converges very quickly to zero, as desired. 

\begin{figure}[t!]
\centering
    \begin{subfigure}{.85\columnwidth}
         \centering
         \includegraphics[trim=100 10 120 10, clip, width=1\textwidth]{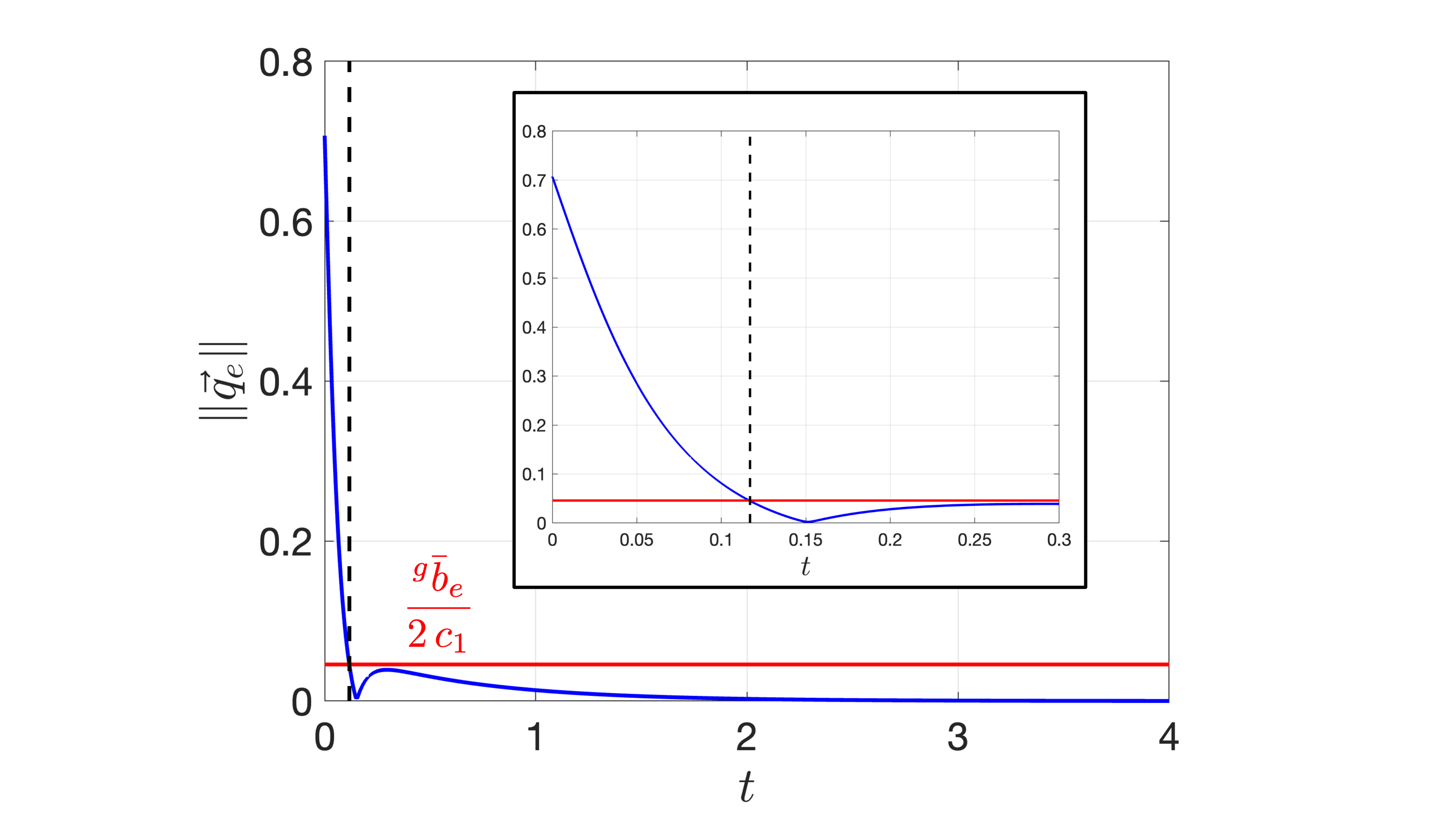}
         \caption{2-norm of $\vec{q}_e$.}
         \label{fig:qe}
     \end{subfigure}
     \\
     \begin{subfigure}{.85\columnwidth}
         \centering
         \includegraphics[ width=1\textwidth]{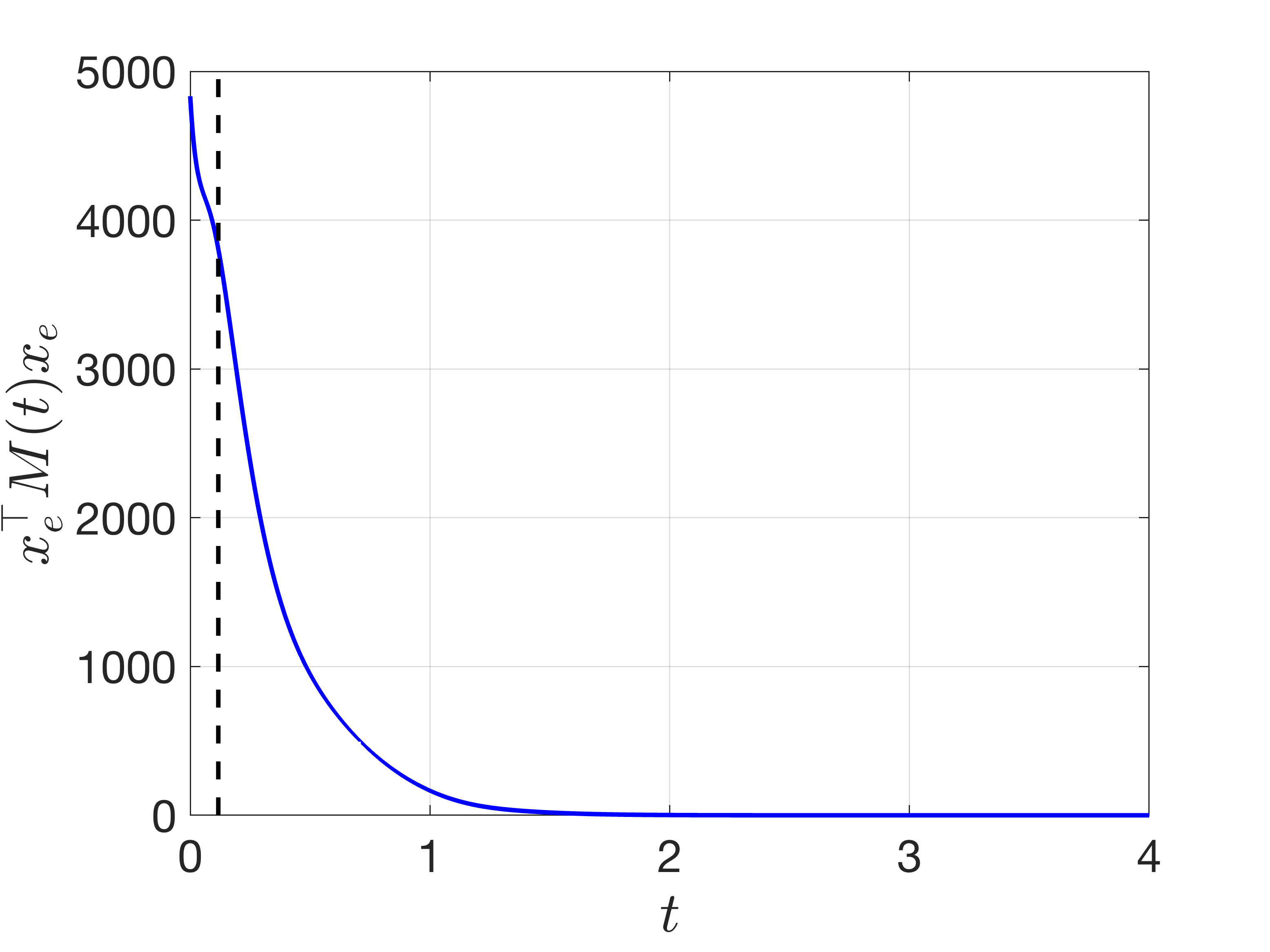}
         \caption{Norm of translation estimation error $x_e$ with metric $M$.}
         \label{fig:E}
     \end{subfigure}
    \label{fig:observer-results}
    \caption{Convergence of the orientation and translation estimation error with the proposed hierarchical observer.}
    \vskip -0.25in
\end{figure}


\section{DISCUSSION}
\label{sec:conclusion}
This work presented a new technique based on contraction analysis to estimate the position, orientation, linear velocity, and IMU bias of a mobile systems equipped with an IMU and other sensors.
The approach relies on the availability of position and orientation measurements generated by an upstream algorithm that uses vision or LiDAR to generate an accurate pose measurement.
The developed approach fuses these pose measurements with high-rate IMU measurements to generate a smooth state estimate --- suitable for control and path planning --- that is guaranteed to converge to its true value.
The approach takes a hierarchical structure where the output of the orientation observer serves as an input to the translation observer.
Future work includes a thorough analysis of the approach when stochastic noise is present in the IMU or pose measurements, and when pose measurements are intermittent. 
Additionally, an initial version of the observer appeared in \cite{chen2022direct} and was shown to perform well in hand-carried experiments.
Extensive hardware experiments including aggressive closed-loop flights on a custom quadrotor will be conducted.


\section*{Appendix}
\label{sec:appendix}
\subsection{Quaternions}
A unit quaternion $q$ is a four vector that lives on the unit three sphere, i.e., $q \in \mathbb{S}^3$, and is composed of a real part $q^\circ$ and vector part $\vec{q}$ such that $ q  = (\,q^\circ,\vec{q}\,)$.
Given two quaternions $p$ and $q$, the quaternion product is
\begin{equation*}
    p \otimes q = \left[\begin{array}{c}
         p^\circ q^\circ - \vec{p}^\top \vec{q} \\ p^\circ\vec{q} + q^\circ\vec{p} + \vec{p} \times \vec{q}
    \end{array}\right].
\end{equation*}
The inverse of quaternion $q$ is its conjugate $q^* \triangleq (\,q^\circ,-\vec{q}\,)$ which satisfies $q \otimes q^* = q^* \otimes q = (\,1,\,\vec{0}\,)$.
A qauternion $q$ can also be converted to rotation matrix $R$, and vice versa, which we denote as $R(q)$.

\subsection{Simulation Parameters}
\begin{table}[!h]
    \centering
    \footnotesize
    \setlength{\tabcolsep}{10 pt}
    \renewcommand{\arraystretch}{1.25}
    \caption{Simulation Parameters.}
    \begin{tabular}{|c|c|}
        \hline
        \textbf{Parameter} & \textbf{Value} \\
        \hline 
        $q(0)$ & [0.7071~0~0.7071~0]$^\top$ \\ \hdashline
        $\hat{q}(0)$ & [1~0~0~0]$^\top$ \\ \hdashline
        $p(0)$ & [0~0~0]$^\top$ m \\ \hdashline
        $\hat{p}(0)$ & [1.68~-1.94~2.01]$^\top$ m \\ \hdashline
        $v(0)$ & [0~0~0]$^\top$ m/s \\ \hdashline
        $\hat{v}(0)$ & [-4.35~1.51~2.44]$^\top$ m/s \\ \hdashline
        $\prescript{g}{}{b}$ & [0.1~-0.02~0.05]$^\top$ rad/s \\ \hdashline
        $\prescript{g}{}{}\hat{b}(0)$ & [0~0~0]$^\top$ rad/s \\ \hdashline
        $\prescript{g}{}{}\bar{b}_e$ & 1.83 rad/s \\ \hdashline
        $\prescript{a}{}{b}$ & [-0.1~0.4~0.2]$^\top$ m/s$^2$ \\ \hdashline
        $\prescript{a}{}{}\hat{b}(0)$ & [0~0~0]$^\top$ m/s$^2$ \\ \hdashline
        $c_1$ & 20 \\ \hdashline
        $c_2$ & 60 \\ \hdashline
        $\lambda$ & 2 \\ \hdashline
        $k_1$ & 64 \\ \hdashline
        $k_2$ & 48 \\ \hdashline
        $k_3$ & 12 \\ \hdashline
        $dt$ & 0.001 s \\
        \hline
    \end{tabular}
    \vspace{-2mm}
    \label{table:params}
\end{table}


\balance
\bibliographystyle{ieeetr}
\bibliography{ref}


\end{document}